\documentclass[12pt]{l4dc2023}

\title[Adaptive Conformal Prediction for Motion Planning among  Dynamic Agents]{Adaptive Conformal Prediction for \\ Motion Planning among  Dynamic Agents}
\usepackage{times}
\usepackage{bbm}
\usepackage{wrapfig}
\usepackage{algorithm}
\usepackage[noend]{algcompatible}
\usepackage{comment}
\newcommand{\mathbold}[1]{\boldsymbol{#1}}
\newtheorem{problem}{Problem}
\usepackage{epstopdf}
\usepackage{hyperref}

\author{%
 \Name{Anushri Dixit}$^{*}$$^{1}$ \Email{adixit@caltech.edu}%
 \AND
 \Name{Lars Lindemann}$^{*}$$^{2}$ \Email{larsl@seas.upenn.edu}%
 \AND
 \Name{Skylar X. Wei}$^{1}$ \Email{swei@caltech.edu}%
 \AND
  \Name{Matthew Cleaveland}$^{2}$ \Email{mcleav@seas.upenn.edu}%
 \AND  
 \Name{George J. Pappas}$^{2}$ 
 \Email{pappasg@seas.upenn.edu}%
 \AND
 \Name{Joel W. Burdick}$^{1}$\Email{jwb@robotics.caltech.edu}\\
 \addr $^{1}$California Institute of Technology, Pasadena, CA \\
 \addr $^{2}$University of Pennsylvania, Philadelphia, PA \\
 \addr $^{*}$ Indicates equal contribution%
}

\begin{document}

\maketitle
\vspace{-14mm}
\begin{abstract}%
This paper proposes an algorithm for motion planning among dynamic agents using adaptive conformal prediction. We consider a deterministic control system and use trajectory predictors to predict the dynamic agents' future motion, which is assumed to follow an unknown distribution. We then leverage ideas from adaptive conformal prediction to dynamically quantify prediction uncertainty from an online data stream. Particularly, we provide an online algorithm that uses delayed agent observations to obtain uncertainty sets for multistep-ahead predictions with probabilistic coverage. These uncertainty sets are used within a model predictive controller to safely navigate among dynamic agents. While most existing data-driven prediction approaches quantify prediction uncertainty heuristically, we quantify the true prediction uncertainty in a distribution-free, adaptive manner that even allows to capture changes in prediction quality and the agents' motion.  We empirically evaluate  our algorithm on a  case study where a drone avoids a flying frisbee.
\end{abstract}

\begin{keywords}%
MPC, dynamic environments, uncertainty quantification, and conformal prediction.
\end{keywords}
\section{Introduction}
Motion planning of autonomous systems in dynamic environments requires the system to  reason about uncertainty in its environment, e.g., a self-driving car needs to reason about uncertainty in the motion of other vehicles, and a mobile robot navigating a crowded space needs to assess uncertainty of nearby pedestrians. These applications are safety critical, as the agents’ intentions are unknown, and systems must be able to plan reactive behaviors in response to an increase in uncertainty. 
 
 Existing works include predictive and reactive approaches, e.g., multi-agent navigation via the dynamic window approach \cite{fox1997dynamic,mitsch2013provably} or navigation functions \cite{dimarogonas2006feedback,tanner2003nonholonomic}. Reactive approaches typically consider simplified dynamics and do not optimize performance. Predictive approaches incorporate predictions of the agents' future motion and can optimize performance. Interactive approaches  take inter-agent interaction into account \cite{kretzschmar2016socially,everett2021collision}, while non-interactive approaches ignore potential interactions \cite{trautman2010unfreezing,du2011robot}.

While many prior works assume perfect knowledge of the environment, an important challenge is to account for uncertainty in perception. Existing works address the problem by making simplifying assumptions, such as linear system dynamics and bounded or Gaussian uncertainty distributions \cite{aoude2013probabilistically,thomas2021probabilistic,renganathan2020towards}. However, addressing the problem in its full generality for nonlinear dynamics and arbitrary distributions is an open problem.

\begin{wrapfigure}[8]{r}{0.3\textwidth}
\vspace{-2ex}
    \includegraphics[scale=0.205]{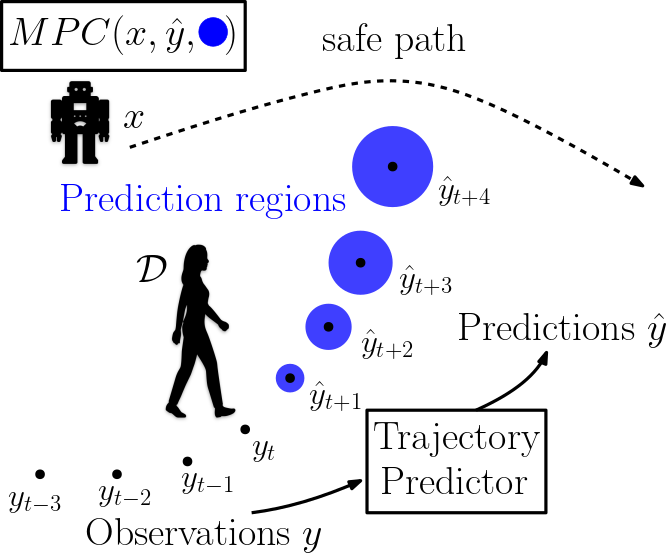}
  \label{fig:intro_figure}
\end{wrapfigure}
 In this paper, we use trajectory predictors to predict the agents’ future motion, and quantify prediction uncertainty in an adaptive and online manner from past agent observations of a single trajectory. Particularly, we use tools from the adaptive conformal prediction (ACP) literature \cite{gibbs2021adaptive,gibbs2022conformal,zaffran2022adaptive,bastani2022practical} to construct prediction regions that quantify multistep-ahead prediction uncertainty. Based on this quantification, we formulate an uncertainty-informed motion planner. Our contributions are as follows:
\begin{itemize}
    \item We propose an algorithm that adaptively quantifies uncertainty of trajectory predictors using ACP. Our algorithm is distribution-free and applies to a broad class of trajectory predictors, providing average probabilistic coverage. 
    \item We propose a model predictive controller (MPC) that leverages uncertainty quantifications to plan probabilistically safe paths around dynamic obstacles. Importantly, our adaptive algorithm enables us to capture and react to changes in prediction quality and the agents’ motion.
    \item We provide empirical evaluations of a drone avoiding a flying frisbee.
\end{itemize}

\subsection{Related Work}


Planning in dynamic environments has found broad interest, and non-interactive sampling-based motion planner were presented in \cite{phillips2011sipp,renganathan2022risk,aoude2013probabilistically,majd2021safe,kantaros2022}, while \cite{du2011robot,lcss_ours,wang2022group,thomas2021probabilistic} propose non-interactive receding horizon planning algorithms. However, accounting for uncertainty in the agent motion is challenging. 

Intent-driven models for planning among human agents have estimated agent uncertainty using Bayesian inference  \cite{fisac2018probabilistically,nakamura2022online,fridovich2020confidence,bansal2020hamilton}. Model predictive control was also used in a stochastic setting to account for uncertainty under the assumption of bounded or Gaussian uncertainty \cite{fan2021step,nairMultimodalMPC2022, yoonGPEKF2021}. Data-driven trajectory predictors can provide mean and variance information of the predictions, which can be approximated as a Gaussian distribution \cite{buschGPpred2022} and used within stochastic planning frameworks \cite{choi2017real,omainska2021gaussian,fulgenzi2008probabilistic}. These approaches quantify prediction uncertainty in a heuristic manner for real systems as the authors make certain assumptions  on prediction algorithms and agent models and its distribution, e.g., being Gaussian. Distributionally robust approaches such as \cite{lcss_ours} are distribution free and can ensure safety at the cost of conservatism.

Data-driven trajectory predictors, such as RNNs or LSTMs, provide no information about prediction uncertainty which can lead to unsafe decisions. For this reason, prediction monitors were recently presented in \cite{farid2022task,luo2021sample} to monitor prediction quality. Especially \cite{luo2021sample}  used conformal prediction to obtain guarantees on the predictor's false negative rate. Conformal prediction was further used to obtain estimates on constraint satisfaction via neural network predictors  \cite{dietterich2022conformal,bortolussi2019neural,qin2022statistical,lindemann2022conformal}. Conceptually closest to our work are \cite{chen2020reactive,lindemann2022safe} where prediction uncertainty quantifications are  obtained using conformal prediction, and then utilized to design model predictive controllers. While the algorithm in \cite{chen2020reactive} can not provide end-to-end safety guarantees,  \cite{lindemann2022safe} can provide probabilistic safety guarantees for the planner. However, changes in the distribution that describes the agents' motion can not be accounted for, e.g., when the agents' motion changes depending on the motion of the control system. Another distinct difference is that offline trajectory data is needed, while we obtain uncertainty quantifications in an adaptive manner from past agent observations of a single trajectory.

\section{Problem Formulation and Preliminaries}
The dynamics of our autonomous system are governed by the discrete-time dynamical system,
\begin{align}
\label{eq:system}
    x_{t+1}=f(x_t,u_t), \;\;\; x_0:=\zeta
\end{align}
where $x_t\in\mathcal{X}\subseteq\mathbb{R}^n$ and $u_t\in \mathcal{U}\subseteq \mathbb{R}^m$ denote the state and the control input at time $t\in\mathbb{N}\cup \{0\}$, respectively. The sets~ $\mathcal{U}$ and $\mathcal{X}$ denote the set of permissible control inputs and the workspace of the system, respectively. The measurable function $f:\mathbb{R}^n\times\mathbb{R}^m\to \mathbb{R}^n$ describes the system dynamics and $\zeta\in\mathbb{R}^n$ is the initial condition of the system. For brevity, let $x:=(x_0,x_1,\hdots)$ denote the trajectory of \eqref{eq:system} under a given control sequence $u:=(u_0,u_1,\hdots)$.

The system operates in an environment with $N$  dynamic agents whose trajectories are a priori unknown. Let $\mathcal{D}(x)$ be an unknown distribution over agent trajectories, i.e., let 
$
    Y:=(Y_0,Y_1,\hdots)\sim \mathcal{D}(x)
$
describe a random trajectory where the joint agent state $Y_t:=(Y_{t,1},\hdots,Y_{t,N})$ at times $t\in\mathbb{N}\cup \{0\}$ is drawn from $\mathbb{R}^{Nn}$, i.e., $Y_{t,j}$ is the state of agent $j$ at time $t$. For instance, $Y_t$ can denote the uncertain two-dimensional positions of $N$ pedestrians at time $t$. Modeling dynamic agents by a distribution $\mathcal{D}$ provides great flexibility, and $\mathcal{D}$ can generally describe the motion of Markov decision processes.  
We use lowercase letters $y_t$ when referring to a  realization of  $Y_t$, and assume at time $t$ to have access to past observations  $(y_0,\hdots,y_t)$.   We make no other assumptions on the distribution $\mathcal{D}$, and in our proposed algorithm we will predict states $(y_{t+1},\hdots,y_{t+H})$ for a prediction horizon of $H$ from $(y_0,\hdots,y_t)$ and quantify prediction uncertainty using ideas from ACP.

\begin{problem}\label{prob1} Given the system in \eqref{eq:system}, the unknown random trajectories $Y\sim\mathcal{D}(x)$, and a failure probability $\delta\in(0,1)$, design the control inputs $u_t$ such that the Lipschitz continuous constraint function $c:\mathbb{R}^n\times\mathbb{R}^{nN}\to \mathbb{R}$ is satisfied\footnote{For an obstacle avoidance constraint, like $c(x,y):= \lVert x - y \rVert -0.5 \geq 0 $, the Lipschitz constant is 1. We implicitly assume that the constraint function is initially satisfied, i.e., that $c(x_0,y_0)\ge 0$.} with a probability of at least $1-\delta$ at each time, i.e., that
\vspace{-0.1cm}
\begin{align}\label{eq:safety_constr}
    \textrm{Prob}\big(c(x_\tau,Y_\tau)\ge 0\big)\ge 1-\delta\;\;\; \text{ for all } \;\;\; \tau\ge 0.
\end{align}
\end{problem}

 We note that our previous work \cite{lindemann2022safe} considers a similar problem formulation. However, in \cite{lindemann2022safe}, we assume that the distribution $\mathcal{D}$ is stationary and it does not depend on the system trajectory $x$ or the environment, i.e., there is no interaction between the control system and the dynamic agents. In reality, however, a pedestrian may come to a halt if a mobile robot comes too close, resulting in a distribution shift in $\mathcal{D}$. This work is a step towards the implementation of a general framework that can adapt to such changes in the agent distribution. 
 
To address Problem \ref{prob1}, we use trajectory predictors to predict the motion of the agents $(Y_0,Y_1,\hdots)$  to enforce the constraint \eqref{eq:safety_constr} within a MPC framework. In \cite{lindemann2022safe}, we assumed the availability of validation data from $\mathcal{D}$ to build prediction regions that quantify uncertainty of trajectory predictors. In this setting, we can collect data online to adapt our uncertainty sets based on past performance of our predictor using ACP without any assumptions on the distribution of the uncertainty and exchangeability of the validation and training dataset.
\begin{remark}
By parameterizing the distribution $\mathcal{D}(x)$ by the trajectory $x$, we model potential interactions between system and  agents. This way, we can adapt to cases where the trajectory predictor (introduced next) is trained without information of $x$, i.e., without taking interactions into account. 
\end{remark}
\textbf{Trajectory Predictors:} Given observations $(y_0,\hdots,y_t)$ at time $t$, we want to predict future states $(y_{t+1},\hdots,y_{t+H})$ for a prediction horizon of $H$. Assume  that \textsc{Predict} is a function that maps observations $(y_{0},\hdots,y_t)$ to predictions $(\hat{y}_t^1,\hdots,\hat{y}_t^H)$ of $(y_{t+1},\hdots,y_{t+H})$. Note that $t$ in $\hat{y}_t^\tau$ denotes the time at which the prediction is made, while $\tau$ indicates how many steps we predict ahead. In principle, \textsc{Predict} can be a classical auto-regressive model or a  neural network based method. 


While our proposed problem solution is compatible with any trajectory predictor \textsc{Predict}, we focus in the case studies on real-time updating strategies like sliding linear predictors with extended Kalman filter. Extracting a dynamics model from data is challenging, especially when the available data is limited, noisy, and partial. 
\cite{takens1981detecting} showed that the method of delays can be used to reconstruct qualitative features of the full-state, phase space from delayed partial observations. By building on our previous work using time delay embedding in dynamic obstacle avoidance (\cite{lcss_ours}), we employ a linear predictor based on spatio-temporal factorization of the delayed partial observations as the pairing trajectory predictor (See Appendix~\ref{appendix:traj_pred}).

\noindent \textbf{Adaptive Conformal Prediction (ACP):} Conformal prediction is used to obtain  prediction regions for  predictive models, e.g., neural networks, without making assumptions on the underlying distribution or the predictive model \cite{vovk2005algorithmic,shafer2008tutorial,angelopoulos2021gentle}. Let $R_1,\hdots,R_{t+1}$ be $t+1$ independent and identically distributed (i.i.d.) random variables. The goal in conformal prediction is to obtain a prediction region of $R_{t+1}$ based on $R_1,\hdots,R_{t}$. Formally, given a failure probability $\delta\in (0,1)$, we want to obtain a  prediction region $C$ such that
\begin{align*}
    \textrm{Prob}(R_{t+1}\le C)\ge 1-\delta.
\end{align*}

We refer to $R_i$ also as the  nonconformity score. For supervised learning, we can select $R_i:=\|Z_i-\mu(X_i)\|$ where $\mu$ is the predictor so that a large nonconformity score indicates a poor predictive model. By a  quantile argument, see \cite[Lemma 1]{tibshirani2019conformal}, we can obtain $C$ to be the $(1-\delta)$th quantile of the empirical distribution of the values $R_1,\hdots,R_{t}$ and $\infty$. Calculating the $(1-\delta)$th quantile can be done by assuming that $\bar{R}_1,\hdots,\bar{R}_{t}$ correspond to the values of $R_1,\hdots,R_{t}$, but instead sorted in non-decreasing order ($\bar{R}$ refers to the order statistic of $R$), i.e., for each $\bar{R}_i$ there exists exactly one $R_j$ such that $\bar{R}_i=R_j$ and $\bar{R}_{i+1}\ge \bar{R}_i$. By setting $q:=\lceil (t+1)(1-\delta)\rceil\le t$, we  obtain the $(1-\delta)$th quantile as $C:=\bar{R}_{q}$, i.e., the $q^{\text{th}}$ smallest nonconformity score. 

The underlying assumption in conformal prediction is that $R_1,\hdots,R_{t+1}$ are exchangeable (exchangeability includes i.i.d. data). This is an unreasonable assumption for time-series prediction where $R_t$ may denote the nonconformity score at time $t$. To address this issue, ACP was introduced in \cite{gibbs2021adaptive,gibbs2022conformal,zaffran2022adaptive,bastani2022practical}. The idea is now to obtain a prediction region $C_{t+1}$ adaptively so that $\textrm{Prob}(R_{t+1}\le C_{t+1})\ge 1-\delta$ for each time $t$. In fact, the prediction region is now obtained as $C_{t+1}:=\bar{R}_{q_{t+1}}$ where $q_{t+1}:=\lceil (t+1)(1-\delta_{t+1})\rceil$ depends on the variable $\delta_{t+1}$ that is adapted online based on observed data.  In this way, the prediction region $C_{t+1}$ becomes a tuneable parameter by the choice of $\delta_{t+1}$. To adaptively obtain the parameter $\delta_{t+1}$, ideas from online learning are used and we update $\delta_{t+1}$ as
\begin{align}\label{eq:adapt_upd_rule}
    \delta_{t+1}:=\delta_{t}+\gamma(\delta-e_{t})
    \; \text{ with }\; e_{t}:=\begin{cases}
    0 &\text{if }\, r_{t}\le C_{t}\\
    1 &\text{otherwise}
    \end{cases}
\end{align}
where we denote by $r_{t}$ the observed realization of $R_{t}$ and where $\gamma$ is a learning rate. The idea is to use $\delta_{t+1}$ to adapt to changes in the distribution of $R_1,\hdots,R_{t+1}$ over time by using information on how much the prediction region $C_{t}$ overcovered ($r_{t}\ll C_{t}$) or undercovered ($r_{t}\gg C_{t}$) in the past. 
\begin{remark}
One of the main performance enhancers is the proper choice of $\gamma$. In \cite{gibbs2022conformal}, the authors present fully adaptive conformal prediction (FACP) where a set of learning rates $\{\gamma_i\}_{1\leq i \leq k}$ is used in parallel from which the best $\gamma$ is selected adaptively. Based on past performance (using a reweighting scheme that evaluates which $\gamma_i$ provided the best coverage), the authors maintain a belief $p_t^{(i)}$ at each time step $t$ for each  $\{\delta_t^{(i)}\}_{1\leq i \leq k}$. The new update laws are
$$
        \delta^{(i)}_{t+1}:=\delta^{(i)}_{t}+\gamma_i(\delta-e^{(i)}_{t})
    \; \text{ with }\; e^{(i)}_{t}:=\begin{cases}
    0 &\text{if }\, r_{t}\le C^{(i)}_{t}\\
    1 &\text{otherwise}
    \end{cases}$$
where the individual prediction regions are $C^{(i)}_{t}:=\bar{R}_{q^{(i)}_{t}}$ with $q^{(i)}_{t}:=\lceil (t+1)(1-\delta^{(i)}_{t})\rceil$, while the best prediction region is  $C_{t}:=\bar{R}_{q_{t}}$ with $q_{t}:=\lceil (t+1)(1-\sum_{i=1}^{k}p_t^{(i)}\delta^{(i)}_{t})\rceil$.
\end{remark}
\section{Adaptive Conformal Prediction Regions for Trajectory Predictions}

Recall that we can obtain predictions  $(\hat{y}_t^1,\hdots,\hat{y}_t^H)$  at time $t$ of future agent states  $(Y_{t+1},\hdots,Y_{t+H})$ from past observations $(y_{0},\hdots,y_t)$ using the \textsc{Predict} function. Note, however, that these point predictions contain no information about prediction uncertainty and can hence not be used to  reason about the safety constraint \eqref{eq:safety_constr}. To tackle this issue, we aim to construct prediction regions for $(Y_{t+1},\hdots,Y_{t+H})$ using ideas from ACP. 

To obtain  prediction regions for $(Y_{t+1},\hdots,Y_{t+H})$,  we could consider the nonconformity score $\|Y_{t+\tau}-\hat{y}_t^\tau\|$ at time $t$ that captures the multistep-ahead prediction error for each $\tau\in\{1,\hdots,H\}$. A large nonconformity score indicates that the prediction $\hat{y}_t^\tau$ of $Y_{t+\tau}$ is not accurate, while a small score indicates an accurate prediction. For each $\tau$, we wish to obtain a prediction region $C_t^\tau$ that is again defined by an update variable $\delta_t^\tau$. Note, however, that we can not evaluate $\|y_{t+\tau}-\hat{y}_t^\tau\|$ at time $t$ as only measurements $(y_0,\hdots,y_t)$ are known, but not $(y_{t+1},\hdots,y_{t+H})$. Consequently, we cannot use the update rule \eqref{eq:adapt_upd_rule} to update $\delta_{t}^\tau$, as the error $e_t^\tau$ would depend on checking if $\|y_{t+\tau}-\hat{y}_{t}^\tau\|\le C_{t}^\tau$. To address this issue, we define the time  lagged nonconformity score
\begin{align*}
    R_t^\tau:=\|Y_{t}-\hat{y}_{t-\tau}^\tau\|
\end{align*}
that we can evaluate at time $t$ so that we can use the update rule \eqref{eq:adapt_upd_rule}. This nonconformity score $R_t^\tau$ is time lagged in the sense that, at time $t$, we evaluate the $\tau$ step-ahead prediction error that was made $\tau$ time steps ago. We can now update the parameter $\delta_{t+1}^\tau$ that defines $C_{t+1}^\tau$ as
\begin{align}\label{eq:recursion}
    \delta_{t+1}^\tau:=\delta_{t}^\tau+\gamma(\delta-e_{t}^\tau)
    \; \text{ with }\; e_{t}^\tau:=\begin{cases}
    0 &\text{if }\, \|y_{t}-\hat{y}_{t-\tau}^\tau\|\le C_{t}^\tau\\
    1 &\text{otherwise.}
    \end{cases}
\end{align}
 
To  compute the prediction region $C_{t+1}^\tau$, note that we can not compute $R_1^\tau,\hdots,R_{\tau-1}^\tau$. Therefore, with minor change, we let $C_{t+1}^\tau$ be the $\lceil(t-\tau+1)(1-\delta_{t+1}^\tau)\rceil^{\text{th}}$ smallest value of $(R_\tau^\tau,\hdots,R_{t}^\tau)$\footnote{Instead of keeping track of all  data, we will choose  a sliding window of the $N$ most recent data. For all prediction regions, we will then consider $(R_{t-N}^\tau,\hdots,R_{t}^\tau)$ and compute $C_{t+1}^\tau$ as the $\lceil(N+1)(1-\delta_{t+1}^\tau)\rceil^{\text{th}}$ smallest value.}. 

By obtaining a prediction region for $R_{t+1}^\tau$ using ACP, we obtain a prediction region for the $\tau$ step-ahead prediction error that was made $\tau-1$ time steps ago, i.e., for $\|Y_{t+1}-\hat{y}_{t+1-\tau}^\tau\|$. Under the assumption that $R_{t+1}^\tau$ and $R_{t+\tau}^\tau$ are independent and identically distributed, $R_{t+1}^\tau$ serves as a prediction region for $\tau$ step-ahead prediction error that was made $0$ time steps ago (now at time $t$), i.e., for $R_{t+\tau}^\tau$ which encodes $\|Y_{t+\tau}-\hat{y}_{t}^\tau\|$. Naturally, in our setting $R_{t+1}^\tau$ and $R_{t+\tau}^\tau$ are not independent and identically distributed, but it still serves as a good measure for the prediction region $R_{t+\tau}^\tau$. We remark that for the theoretical guarantees that we provide in the next section, only the one step-ahead prediction errors are relevant. 
\begin{theorem}\label{thm:1}
   Let $\gamma$ be a learning rate, $\delta_0^1\in (0,1)$ be an initial value for the recursion \eqref{eq:recursion}, and $T$ be the number of times that we compute the recursion \eqref{eq:recursion}. Then, for the onestep-ahead prediction errors, it holds that
     \begin{align}\label{eq:thm1}
 1-\delta-p_1 \leq \frac{1}{T}\sum_{t=0}^
{T-1} \textrm{Prob}(\|Y_{t+1}-\hat{y}_{t}^1\|\le C_{t+1}^1)\leq 1-\delta+p_2
     \end{align}
     with constants $p_1:=\frac{\delta_0^1+ \gamma}{T\gamma}$, $p_2:=\frac{(1-\delta_0^1)+ \gamma}{T\gamma}$ so that $\lim_{T\rightarrow\infty}p_1 =0$ and $\lim_{T\rightarrow\infty}p_2 =0$.
\end{theorem}
\begin{proof}
     Since the probability of an event is equivalent to the expected value of the indicator function of that event, it follows by the definition of the error $e_{t+1}^1$ that
\begin{align}\label{eq:proof1}
    \textrm{Prob}(\|Y_{t+1}-\hat{y}_{t}^1\|\le C_{t+1}^1)= \mathbb{E}[1-e_{t+1}^1] = 1-\mathbb{E}[e_{t+1}^1].
\end{align}

For a given initialization $\delta_0^\tau$ and learning rate $\gamma$, we know from~\cite[Proposition 4.1]{gibbs2021adaptive} that the following bound holds (with probability one) for the misclassification errors
\begin{align*}
\;\;\;\frac{-(1-\delta_0^1)+ \gamma}{T\gamma}&\leq \frac{1}{T}\sum_{t=0}^
{T-1} e_{t+1}^1 -\delta \leq \frac{\delta_0^1+ \gamma}{T\gamma}
\implies \Big|\frac{1}{T}\sum_{t=0}^
{T-1} e_{t+1}^1 -&\delta\Big| \leq \frac{\max(\delta_0^1,1-\delta_0^1)+ \gamma}{T\gamma}.
\end{align*}
Hence, taking the expectation of the above two-sided inequality, we get that
\begin{align*}
    \frac{-(1-\delta_0^1)+ \gamma}{T\gamma} &\leq \frac{1}{T}\sum_{t=0}^
{T-1} \mathbb{E}[e_{t+1}^1] -\delta \leq \frac{\delta_0^1+ \gamma}{T\gamma},\\
\overset{(a)}{\Leftrightarrow}\;\;\;\frac{-(1-\delta_0^1)+ \gamma}{T\gamma} &\leq\frac{1}{T}\sum_{t=0}^
{T-1} \big(1-\textrm{Prob}(\|Y_{t+1}-\hat{y}_{t}^1\|\le C_{t+1}^1)\big)-\delta \leq  \frac{\delta_0^1+ \gamma}{T\gamma},\\
\Leftrightarrow\;\;\;1-\delta + \frac{(1-\delta_0^1)+ \gamma}{T\gamma} &\geq \frac{1}{T}\sum_{t=0}^
{T-1} \textrm{Prob}(\|Y_{t+1}-\hat{y}_{t}^1\|\le C_{t+1}^1) \geq 1-\delta - \frac{\delta_0^1+ \gamma}{T\gamma},
\end{align*}
where we used equation \eqref{eq:proof1} for the equivalence in (a).
\end{proof}
\begin{remark}
The above result can be similarly extended to the FACP case with a set of candidate learning rates, $\gamma$, \cite[Theorem 3.2]{gibbs2022conformal}.
\end{remark}
\begin{example}
To illustrate these multistep-ahead prediction regions, consider a planar double pendulum whose dynamics are governed by chaotic, nonlinear dynamics that are sensitive to the initial condition~\cite[]{shinbrotDP1992}.
We study the predictions made by a linear predictor that uses noisy observations of the position of the double pendulum (See Appendix~\ref{appendix:traj_pred})
and use ACP to predict the uncertainty in the predictions. Both the trajectory predictor and the uncertainty quantification using ACP use online data from a single trajectory. ACP provides the multi-step errors in the linear predictions with a coverage level of $\delta = 0.1$, and learning rates $\gamma = \begin{pmatrix}
    0.0008 & 0.0015 & 0.003 & 0.005& 0.009 & 0.017 &0.03 & 0.05 &0.08
\end{pmatrix}.$

Figure~\ref{fig:multi_step_DP} compares the 1-step and 6-step ahead error prediction regions to the true multi-step errors for two states, the second mass position, $x_2, y_2$. The percentages of one-step errors that are incorrectly predicted, i.e., $e_t^1 = 1$, for the positions of each mass, $x_1, x_2, y_1, y_2$ are $2.36\%, 0.94\%, 1.57 \%, 1.73\%$  respectively. We can see the effects of adaptation as the ACP prediction regions are larger in areas of poor performance of the linear predictor (and consequently higher error in the prediction) and smaller in regions where the linear predictor performs well.

\begin{figure*}[h!]
    \centering{
\includegraphics[width=0.85\linewidth]{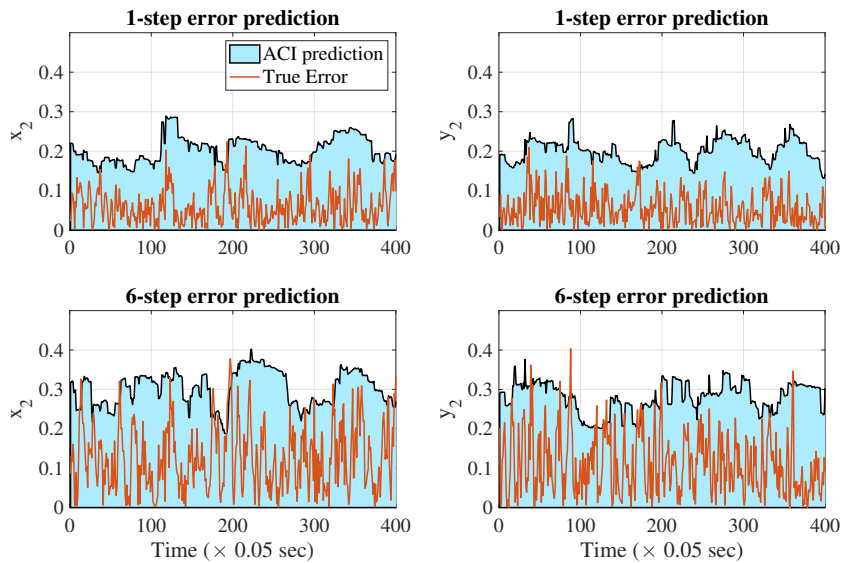}
    \vspace{-5mm}
    \caption{The multi-step prediction errors are shown for two of the six states of a double pendulum ($x_2, y_2$). ACP can correctly predict regions of high and low error ($90\%$ coverage regions) by adjusting the prediction quantile using update law (\ref{eq:adapt_upd_rule}). The orange lines are the true multi-step prediction errors and the blue areas are the error regions predicted by ACP.}}
    \vspace{-8mm}
    \label{fig:multi_step_DP}
\end{figure*}
\end{example}

\section{Uncertainty-Informed Model Predictive Control}

Based on the obtained uncertainty quantification from the previous section, we propose an uncertainty-informed model predictive controller (MPC) that uses predictions $\hat{y}_t^\tau$ and adaptive prediction regions $C_{t+1}^\tau$. The underlying optimization problem that is solved at every time step $t$ is:
\begin{subequations}\label{eq:open_loop}
\begin{align}
    \min_{(u_t,\hdots,u_{t+H-1})}& \sum_{k=t}^{t+H-1}J(x_{k+1},u_k) &\\
     \text{s.t.}\qquad & x_{k+1}=f(x_k,u_k), &k\in\{t,\hdots,t+H-1\}\\
     & c(x_{t+\tau},\hat{y}_t^\tau)\ge LC_{t+1}^\tau,&\tau\in\{1,\hdots,H\}\label{eq:constC_2}\\
     & u_k \in \mathcal{U},x_{k+1} \in \mathcal{X},&k\in\{t,\hdots,t+H-1\} 
\end{align}
\end{subequations}
where $L$ is the Lipschitz constant of the constraint function $c$, $J$ is a step-wise cost function, and  $u_t,\hdots,u_{t+H-1}$ is the control sequence. The optimization problem in \eqref{eq:open_loop} is convex if the functions $J$ and $f$ are convex, the function $c$ is convex in its first argument, and the sets $\mathcal{U}$ and $\mathcal{X}$ are convex. 

Based on this optimization problem, we propose a receding horizon control strategy in Algorithm \ref{alg:overview}. In line 1 of Algorithm \ref{alg:overview}, we initialize the parameter $\delta_0^t$ simply to $\delta$. Lines 2-11 present the real-time planning loop by: 1) updating the states $x_t$ and $y_t$ and calculating new predictions $\hat{y}_{t}^\tau$ (lines 3-4), 2) computing the adaptive nonconformity scores $C_{t+1}^\tau$ (lines 5-9), and 3) solving the optimization problem in \eqref{eq:open_loop} of which we apply only $u_t$ (lines 10-11). 

\begin{algorithm}
    \centering
    \begin{algorithmic}[1]
        \Statex \textbf{Input: } Failure probability $\delta$, prediction horizon $H$, learning rate $\gamma$ 
        \Statex \textbf{Output: } Control input $u_t(x_t,y_0,\hdots,y_t)$ at each time $t$
        \State $\delta_0^\tau \gets \delta$ for $\tau\in\{1,\hdots,H\}$
        \FOR{$t$ from $0$ to $\infty$} \quad\# real-time motion planning loop
            \State Update $x_t$ and  $y_t$ 
            \State Obtain predictions $\hat{y}_t^\tau$ for $\tau\in\{1,\hdots,H\}$
            \FOR{$\tau$ from $1$ to $H$} \quad\# compute ACP regions
                \State $\delta_{t+1}^\tau \gets \delta_{t}^\tau+\gamma(\delta-e_{t}^\tau)$ 
                \State $R_{t}^\tau:=\|y_t-\hat{y}_{t-\tau}^\tau\|$
                \State $q \gets \big\lceil (t+1)(1-\delta_{t+1}^\tau)\big\rceil$
                \State Set $C_{t+1}^\tau$ as the $q$th smallest value of $(R_\tau^\tau,\hdots,R_t^\tau)$
            \ENDFOR
            \State Calculate controls $u_t,...,u_{t+H-1}$ as the solution of \eqref{eq:open_loop}
            \State Apply $u_t$ to \eqref{eq:system}
        \ENDFOR
    \end{algorithmic}
    \caption{MPC with ACP Regions}
    \label{alg:overview}
\end{algorithm}
\begin{remark}
    While Algorithm~\ref{alg:overview} uses a single learning rate, one can similarly extend the above algorithm to be fully adaptive using a candidate set of $\{\gamma_i\}_{1\leq i\leq k}$ without loss of generality.
\end{remark}
\begin{remark}
\cite{gibbs2021adaptive} assume that when $\delta_{t+1} \leq 0$, the prediction region $C_{t+1}\rightarrow \infty$. This means that when the algorithm requires robust behavior, the $\infty$-prediction region ensures that any prediction at the next time-step should be correctly classified. For a physical system, there are limits on how much the dynamic obstacle can accelerate in one time-step which gives us an upper bound $R_{\text{max}}< \infty$ on the worst-case error. In practice, we enforce $0\leq\delta_{t+1}\leq 1$ with $C_{t+1}\leq R_{\text{max}}$.
\end{remark}
\begin{theorem}\label{thm:2}
     Let $\gamma$ be a learning rate, $\delta_0^1\in (0,1)$ be an initial value for the recursion \eqref{eq:recursion}, and $T$ be the number of times that we compute the recursion \eqref{eq:recursion}. If the optimization problem \eqref{eq:open_loop} in Algorithm~\ref{alg:overview} is recursively feasible, then Algorithm~\ref{alg:overview} will lead to
     \begin{align}
         \frac{1}{T}\sum_{t=0}^
{T-1}\textrm{Prob}\big(c(x_{t+1},Y_{t+1})\ge  0\big) \geq 1-\delta - p_1
     \end{align}
     with constant $p_1:=\frac{\delta_0^1+ \gamma}{T\gamma}$ so that $\lim_{T\rightarrow\infty}p_1 =0$.
\end{theorem}
\begin{proof}
By assumption, the optimization problem in \eqref{eq:open_loop} is feasible at each time $t\in \{0,1,\hdots\}$. Due to constraint \eqref{eq:constC_2} and Lipschitz continuity of $c$, it hence holds that
\begin{align}\label{eq:proof2}
    0&\le c(x_{t+1},\hat{y}_{t}^1)- LC_{t+1}^1\le c(x_{t+1},Y_{t+1})+L\|Y_{t+1}-\hat{y}_{t}^1\|- LC_{t+1}^1
\end{align}
at each time $t\in \{0,1, \hdots\}$. Consequently, note that $\|Y_{t+1}-\hat{y}_{t}^1\|\le C_{t+1}^1$ is a sufficient condition for  $c(x_{t+1},Y_{t+1})\ge 0$. In a next step, we can derive that  
\begin{align*}
    \textrm{Prob}\big(c(x_{t+1},Y_{t+1})\ge 0\big) &\overset{(a)}{=}\textrm{Prob}\big(c(x_{t+1},Y_{t+1})\ge 0 \,\big\vert \, \lVert Y_{t+1}-\hat{y}_{t}^1\rVert\le C_{t}^1\big)\textrm{Prob}(\lVert Y_{t+1}-\hat{y}_{t}^1\rVert\le C_{t}^1)  \\ &\qquad \hspace{-0.75cm}+\textrm{Prob}\big(c(x_{t+1},Y_{t+1})\ge 0 \,\big\vert \, \lVert Y_{t+1}-\hat{y}_{t}^1\rVert> C_{t}^1\big)\textrm{Prob}(\lVert Y_{t+1}-\hat{y}_{t}^1\rVert> C_{t}^1) \\
    &\overset{(b)}{\geq}\textrm{Prob}\big(c(x_{t+1},Y_{t+1})\ge 0  \,\big\vert \, \lVert Y_{t+1}-\hat{y}_{t}^1\rVert\le C_{t}^1\big)\textrm{Prob}(\lVert Y_{t+1}-\hat{y}_{t}^1\rVert\le C_{t}^1) \\
    &\overset{(c)}{=} \textrm{Prob}(\lVert Y_{t+1}-\hat{y}_{t}^1\rVert\le C_{t}^1)
\end{align*}

\noindent where the equality in (a) follows from the law of total probability, while the inequality in (b)  follows from the nonnegativity of probabilities. The equality in (c) follows as $\textrm{Prob}(c(x_{t+1},Y_{t+1})\ge 0  \,\vert \, \lVert Y_{t+1}-\hat{y}_{t}^1\rVert\le C_{t}^1)=1$ since $\lVert Y_{t+1}-\hat{y}_{t}^1\rVert\le C_{t}^1$ implies $c(x_{t+1},Y_{t+1})\ge 0$ according to  \eqref{eq:proof2}. We now use the result from Theorem~\ref{thm:1} to complete the proof.
\end{proof}

\vspace{-5mm}
\section{Case Studies: Multirotor operating in small angle regime dodging a flying frisbee}

We compare the performance of the MPC with ACP uncertainty prediction regions with our past work that uses a distributionally robust approach to uncertainty quantification \cite[]{lcss_ours}. We use the same multirotor operating in the presence of a moving obstacle example with a MPC planner. The multirotor is constrained to operate within the state constraints
$\theta\in [-0.45,0.45]$~radians and $\varphi \in [-0.45,0.45]$ ~radians. We use the following standard multirotor linear dynamics,
\begin{align} \label{eq:sim_agent_dyn2}
 \ddot{x}= -g\theta, \, \,
       \ddot{y} = g\varphi,\,\, \ddot{z}=u_1 -g, \,\,
       \ddot{\varphi} = \frac{u_2}{I_{xx}},\,\, \ddot{\theta} = \frac{u_3}{I_{yy}},\,\, \ddot{\psi} = \frac{u_4}{I_{zz}},\vspace{-3mm}
\end{align}
\noindent where the planner control inputs $u_1, u_2, u_3, u_4$ correspond to the thrust force in the body frame and three moments. The vehicle's moments of inertia are $I_{xx} = 0.0075 kgm^2, I_{yy} = 0.0075 kgm^2, I_{zz} = 0.013kgm^2$. The MPC planner has a horizon length of 10 steps and the planner is updated at 20 Hz. It is implemented through a Sequential Convex Programming approach \cite[]{SCP1}.

Numerical simulations of the proposed MPC planner with ACP regions and dynamics \eqref{eq:sim_agent_dyn2} are presented as it avoids a Frisbee that is thrown at the drone from various initial positions, velocities, and rotation speed. The Frisbee is modeled following~\cite{hummel2003frisbee}, and we implement linear predictions of the trajectory arising from its nonlinear dynamics.

We conducted 1000 Monte Carlo simulations per allowed failure probability level $\delta$ to compare the numerical feasibility, percentage of success in obstacle avoidance (if the MPC planner is feasible), and the planner's conservativeness, as measured by the minimum distance between the obstacle and agent centers, i.e., $\bar{d}_{min}$ and  $\sigma({d_{min}})$ describe the average and standard deviation of this minimum distance across simulations, respectively. We compare three uncertainty quantification techniques in Table~\ref{table:table1}, (1) The proposed ACP method (Algorithm \ref{alg:overview}), (2) empirical bootstrap prediction that accounts for the uncertainty in the predictions using the empirical bootstrap variance \cite[]{lcss_ours}, and (3) the sliding linear predictor with an Extended Kalman Filter (EKF) that approximates the uncertainty in the obstacle predictions as a Gaussian distribution (See Appendix~\ref{appendix:traj_pred}).

\textbf{Discussion: } Table~\ref{table:table1} shows that our proposed method can successfully avoid the Frisbee, while using a significantly smaller average divergence distance ($d_{min}, \,\sigma({d_{min}})$) from the Frisbee. I.e., our approach avoids the conservatism of other approaches due to the adaptivity of the uncertainty sets. 
Our method can usefully adjust the prediction sets when the underlying uncertainty distribution is shifting (due to discrepancy in the linear dynamic predicted and the true nonlinear obstacle motion). We also note that the feasibility of the MPC optimization is worse for our method compared to~\cite{lcss_ours} and the EKF predictor. This issue arises during sudden changes in the size of the uncertainty sets when the learning rate $\gamma$ is chosen too large. We will investigate this issue in future work by considering tools to ensure recursive feasibility~\cite[]{HEWING2020109095} or by providing backup controllers~\cite[]{drew2022backup,jesus2020faster} when the MPC is infeasible.
\begin{table}
\begin{center}
\footnotesize
\begin{tabular}{ |c|c|ccc|ccc| } \hline 
  Case& $\delta$ &  &$0.025$&& &$0.05$&  \\ 
   &  UQ method &  Proposed & \cite{lcss_ours} & w/EKF & Proposed  & \cite{lcss_ours} & w/EKF \\\hline \hline
 & $\%$Feas.                &83.8&87.4&   97.1& 80.9&90.3&97.6  \\ 
Frisbee & $\%$Succ.    &99.2&100& 100 &100&100&100    \\ 
w/drag & $\overline{d}_{min}$ &2.91&14.2& 5.27 &2.74&4.97&4.25   \\ \
&$\sigma(d_{min})$            &1.25&2.04& 1.28 &1.3&1.97&1.11    \\ \hline 
\end{tabular}
\caption{Summary of results from MC simulations of system \eqref{eq:sim_agent_dyn2}. We used FACP for predicting uncertainty sets with learning rates $\gamma = \{0.0008,\, 0.0015,\, 0.003,\, 0.005,\, 0.009,\, 0.017,\,0.03,\,0.05,\,0.08,\, 0.13 \} $ and using the last $30$ measurements of the obstacle.}
\label{table:table1}
\normalsize
\end{center}
\vspace{-10mm}
\end{table}

\section{Conclusion}
We presented an algorithm for safe motion planning in an environment with other dynamic agents using ACP. Specifically, we considered a deterministic control system that uses state predictors to estimate the future motion of dynamic agents. We then  leveraged ideas from ACP to dynamically quantify prediction uncertainty from an online data stream, and designed an uncertainty-informed model predictive controller to safely navigate among dynamic agents. In contrast to other data-driven prediction models that quantify prediction uncertainty in a heuristic manner, we quantify the true prediction uncertainty in a distribution-free, adaptive manner that even allows to capture changes in prediction quality and the agents' motion.  
\acks{Lars Lindemann, Matthew Cleaveland, and George J. Pappas were generously supported by NSF award CPS-2038873. The work of Anushri Dixit and Skylar Wei was supported in part by DARPA, through the LINC program.}

\bibliography{literature}
\appendix
\section{Sliding linear predictor with Extended Kalman filter}~\label{appendix:traj_pred}

Given observations $\{y\}_{0}^t$ at the current time $t\ge 0$ of a discrete-time multivariate stochastic process, we assume the agent is governed by 
\begin{equation}
    z_{t+1} = f_d^{agent}(z_{t}),\quad \quad  y_t = h_d(z_{t}) + \xi_t
\end{equation}
where $f_d$ and $h_d$ are smooth (infinitely differentiable) functions that are the unknown states transition function in terms of full agent state $z_t$ and observation function that maps $z_t$ into observables (partial) $\{y\}_{0}^{t}$, respectively. The observables are corrupted by independent identically distributed Gaussian noise $\{\xi\}_{0}^{t}$ where each $\xi_i \in \mathcal{N}(0,\sigma^2)$.

Our goal is to obtain predictions  $(\hat{y}_t^1,\hdots,\hat{y}_t^H)$  at time $t$ of future agent states  $(Y_{t+1},\hdots,Y_{t+H})$ from past observations $(y_{0},\hdots,y_t)$ using the \textsc{Predict} function.

Together, we let the vector $g_{0:L-1}^{(i)}\triangleq[
y_{0}^{(i)},y_{1}^{(i)},\cdots,y_{L-1}^{(i)}
]^{T} \in \mathbb{R}^{L}$ be the $L$-delay embedding the of $i^{th}$ measurable. As additional observable is acquired as time progresses, we can construct the trajectory matrix of the $i^{th}$ observable $\{y^{(i)}_0,\cdots,y^{(i)}_{N}\}$ or also known as the Hankel matrix:
\begin{align} \label{eq:hankel}
    H^{(i)}_{[L,N]} = \begin{bmatrix}
        g^{(i)}_{0:L-1} & g^{(i)}_{1:L} &\cdots  & g^{(i)}_{L:N-1}
        \end{bmatrix} = U\Sigma V^*
\end{align}
The matrix of left singular vectors $U = \begin{bmatrix}{\mu}_1, \cdots, {\mu}_L \end{bmatrix}$ is orthonormal. The principal components of $H^x_{L\times N}$ are the columns of $V$. 
To efficiently separate the noise and the true signal, we follow the work by \cite{variant_mssa} introduce the Page matrix representation of observables $\{y_0,\cdots,y_{TL-1}\}$. We construct and denote a $L$-embedding Page matrix as $P^{(i)}_{[L,TL]}$:
\begin{equation}
    P^{(i)}_{[L,TL]} = \begin{bmatrix} g^{(i)}_{0:L-1} & g^{(i)}_{L:2L-1} &\cdots  & g^{(i)}_{(T-1)L:TL-1}
    \end{bmatrix}=  U_P \sigma_P V_P^*
\end{equation}
Unlike Hankel matrices \eqref{eq:hankel}, Page matrices do not have repeated entries which enable us to leverage the result by \cite{optimalhardthreshold}, an optimal hard singular value threshold (optHSVT) algorithm (with respect to the Mean Squared Error) for any unknown $m\times n$ matrix corrupted by noise that is zero mean, identically and independently distributed.
In summary, the optHSVT algorithm provides $\sigma_{HSVT}$ that partitions the Page matrix as,
\begin{equation} \label{eq:page_noise_signal_separation}
    P^{(i)}_{[L,TL]} = \underbrace{\sum_{\rho = 1}^{n_{HSVT}} \sigma_\rho \mathbold{\mu}_\rho\mathbold{\nu}_\rho^T}_{\approx\,\, \mbox{signal}} + \underbrace{\sum_{\rho=n_{HSVT}+1}^{\min\{L, T\}}\sigma_\rho \mathbold{\mu}_\rho\mathbold{\nu}_\rho^T}_{\approx\,\, \mbox{noise}}.
\end{equation}
where $n_{HSVT}$ is the index of singular value of the logic statement $\sigma_\rho\left( P^{(i)}_{[L,TL]}\right) \geq \sigma_{HSVT}$.
Since both Hankel and Page construction shares the same rank, allowing us to use the Page matrix to recover the rank of the system and avoid ill-conditioned matrix inversion. As a result, we extract a linear predictor using the pseudo inverse $\Lambda_t = \hat{H}_{[L,2L]}^{(i),2:L} (\hat{H}_{[L,2L]}^{(i),1:L-1})^{\dagger}$ (similar to the minimum linear recurrence result in \cite{SSAbook}). In particular, we denote $$\hat{H}^{(i),2:L}_{[L,2L]} = \Big[ g^{(i)}_{t-2L:t-L-1},  g^{(i)}_{t-2L+1:t-L}, \cdots,  g^{(i)}_{[t-L:t-1]}\Big],$$  $$\hat{H}^{(i),1:L-1}_{[L,2L]} = \Big[ g^{(i)}_{t-2L+1:t-L}, g^{(i)}_{t-2L+1:t-L+1},\cdots, g^{(i)}_{[t-L+1:t]}\Big],$$ which are both $L\times L$ matrices. The $\hat{\cdot}$ operation is reconstructing the Hankel matrices with the first $n_{HSVT}$ singular eigenvector and eigenvalue pairs.  At each instance, a $H$ step prediction simply extracts the last $H$ elements of the last column of $(\Lambda_t)^{H}\hat{H}^{(i),2:L}_{[L,2L]}$. This linear predictor model will be updated instantly as new measurements are obtained. Further, we employ a standard Extended Kalman Filter (EKF) which allows us the incorporate the new measurements observed over time where the instantaneous $\Lambda_t$ is approximated as the Jacobian of the state transition function of the agent. 



\end{document}